%% file: main.tex
\input{config_preprint}

\renewcommand{\hl}[1]{#1}

\begin{document}

\maketitle

\input{abstract}
\input{introduction}
\input{background}

\input{formulation}
\input{results}
\input{conclusion}

\section*{Data Availability}

The data that support the findings of this study are openly available at the following URL/DOI: \href{https://doi.org/10.5281/zenodo.22257448}{https://doi.org/10.5281/zenodo.22257448}~\cite{morrow2026enablingzenodo}.

\section*{Acknowledgements}

The work was supported by the U.S. Department of Energy's Office of Cybersecurity, Energy Security, and Emergency Response (Z.M., J.C., and D.J.K.), as well as the Laboratory Directed Research and Development program at Sandia National Laboratories (Z.M. and J.D.J.).

Sandia National Laboratories is a multimission laboratory managed and operated by National Technology and Engineering Solutions of Sandia, LLC, a wholly owned subsidiary of Honeywell International Inc., for the U.S. Department of Energy's National Nuclear Security Administration under contract DE-NA0003525. This paper describes objective technical results and analysis and has been authored by an employee of Technology \& Engineering Solutions of Sandia, LLC. Any subjective views or opinions that might be expressed in the paper do not necessarily represent the views of the U.S. Department of Energy (DOE) or the United States Government. The employee owns all right, title, and interest in and to the article and is solely responsible for its contents. The United States Government retains and the publisher, by accepting the article for publication, acknowledges that the United States Government retains a non-exclusive, paid-up, irrevocable, world-wide license to publish or reproduce the published form of this article, or allow others to do so, for United States Government purposes. The DOE will provide public access to the results of federally sponsored research in accordance with the DOE Public Access Plan (\url{https://www.energy.gov/downloads/doe-public-access-plan}). SAND2026-20737O

\input{appendices}

\def\bibfont{\small}
\setlength{\bibitemsep}{0em}
\printbibliography[heading=bibliography]

\end{document}

%% file: config_preprint.tex
\documentclass[11pt]{article}
\usepackage{graphicx} 
\usepackage[english]{babel}
\usepackage[T1]{fontenc}
\usepackage[utf8]{inputenc}
\usepackage{csquotes}
\usepackage{amssymb, amsbsy, amsthm, amsmath}
\usepackage[left=1in, right=1in, top=1in, bottom=1in]{geometry}
\usepackage{mathtools}
\usepackage{mathrsfs}
\usepackage{bm}
\usepackage{dsfont}
\usepackage{float}
\usepackage[colorlinks=true,
citecolor=blue,
urlcolor=blue,
linkcolor=blue]{hyperref}
\usepackage{subcaption}
\usepackage{soul}

\usepackage{caption}

\usepackage[toc,page]{appendix}
\usepackage[capitalize, nameinlink]{cleveref}
\crefformat{equation}{(#2#1#3)}
\newcommand{\crefrangeconjunction}{--}
\crefrangeformat{equation}{(#3#1#4)\crefrangeconjunction(#5#2#6)}
\crefname{app}{Appendix}{Appendices}
\Crefname{app}{Appendix}{Appendices}

\renewcommand\vec{\bm}
\newcommand\R{\mathbb{R}}
\newcommand\N{\mathbb{N}}
\newcommand\op{\mathscr{A}}
\newcommand\opinputspace{\mathcal{X}}
\newcommand\opoutputspace{\mathcal{Y}}
\newcommand\opinput{f}
\newcommand\opoutput{g}
\newcommand\basisinput{\phi}
\newcommand\basisoutput{\varphi}
\newcommand\ip[2]{\langle #1, #2 \rangle}
\renewcommand\d{\text{d}}
\newcommand\opspace{\mathcal{G}}
\newcommand\surrclass{\mathcal{C}}

\newcommand{\citeA}{\cite}
\newtheorem{theorem}{Theorem}[section]

\DeclareMathOperator*{\argmax}{arg\, max}

\newcommand{\zacknote}[1]{{\bf \color{red} [ZM: {#1}]}}

\usepackage[backend=bibtex, 
citestyle=numeric-comp,
bibstyle=numeric-comp,
sorting=nyt,
doi=true,
isbn=false,
maxnames=99,
giveninits=true,
uniquename=init,
url=true]{biblatex}
\renewbibmacro{in:}{}
\DefineBibliographyExtras{english}{%
	\uspunctuation%
}
\DeclareFieldFormat[misc,techreport,report]{title}{\mkbibquote{#1}}

\title{Enabling Real-Time Training of a Wildfire-to-Smoke Map with Multilinear Operators}
\author{Zachary Morrow,$^*$ Joseph Crockett, John D. Jakeman, Dan J. Krofcheck}
\date{\footnotesize \em Sandia National Laboratories \\[0em] 1515 Eubank Blvd SE, Albuquerque, NM 87123 \\[0em] $^*$ Corresponding author: \textrm{\tt zbmorro@sandia.gov}}

%% file: abstract.tex
\begin{abstract}
    Wildfires are a major producer of fine particulate matter, impacting human health and the electrical grid. Accurately forecasting smoke impacts over long time scales incorporates fuel treatment strategies, natural fuel succession, and stochastic events like lightning strikes. However, predicting smoke for each fuel distribution with a forward simulation of a coupled fire--atmosphere model is computationally infeasible. Moreover, relatively simple fire models are tractable to run in many long-time scenarios but do not capture smoke transport. We use data-driven multilinear operators to predict a smoke concentration field from knowledge of the time since ignition for two quantities of interest: aerosol optical depth and smoke detection. Our method first computes the principal components of time-since-ignition and smoke concentration fields and then learns a map from powers of the input coefficients to the output coefficients. We apply our learned operator to smoke prediction in the Upper Rio Grande Watershed. After collecting training data, learning the approximation weights on a CPU takes less than 30 seconds, and each forward call takes less than 1 ms. On a proxy for aerosol optical depth, we obtain equal accuracy to Monte Carlo sampling with fewer than half as many coupled model calls. For smoke detection, we obtain an intersection-over-union (IoU) of 0.64 and an area under the receiver operating characteristic curve (AUC) of 0.95 on holdout data. Our method is significantly more accurate than the most similar published smoke classifier, which obtains an IoU and AUC of 0.15 and 0.61, respectively, on a 2015 bushfire in Australia.
\end{abstract}

%% file: introduction.tex
\section{Introduction} \label{sec:intro}

Wildfires are an increasing threat to life, property, and the environment, with average fire size, frequency, and extent rising sharply across the western United States in recent decades \cite{iglesias2022us, abatzoglou2016impact}. Increases in temperature and vapor pressure deficit have enhanced fuel aridity, roughly doubling the cumulative western US forest area burned since the 1980s \cite{abatzoglou2016impact}. In addition to the damage from the fire itself, smoke impacts can extend far away from the burn area. Fine particulate matter under 2.5 microns in diameter (PM2.5) damages human health \cite{ma2024long, burke2023contribution} and burdens the electrical grid through deposits on power lines \cite{jazebi2020review} and through aerosol-induced reductions in photovoltaic output \cite{ford2024quantifying, corwin2025solar}. Smoke aerosols can additionally alter cloud microphysics and precipitation patterns \cite{twohy2021biomass}, with downstream hydrological consequences. Over decadal time-scales, fire behavior is difficult to predict because fuel succession, fire weather, and ignition regimes are all evolving simultaneously \cite{abatzoglou2016impact, balch2017human}. In southwestern US ecosystems such as the Upper Rio Grande Watershed (the focus of this study), a century of fire suppression has further produced fuel loads outside the historical range, exacerbating burn severity in landscapes that historically experienced frequent low-severity surface fire \cite{allen2002ecological}. Our ultimate goal is to predict electrical output reductions across many plausible fuel-management and climate scenarios over decadal time horizons. The slow-variable evolution (fuel succession, climate) is handled by ensemble sampling of fuel snapshots produced by a fuel-succession model; for each snapshot, our surrogate replaces an event-scale coupled fire--atmosphere simulation, mapping a time-since-ignition field to a cumulative smoke concentration field. From smoke concentration, power-system output reductions can in turn be related to aerosol optical depth (AOD) via published empirical relationships \cite{ali2023data, ford2024quantifying}. We work with the passive bulk smoke tracer reported by WRF-Chem as a proxy for total particulate emissions; speciation into PM2.5 and other constituents, conversion to AOD, and quantitative coupling to PV output are layered on top in subsequent work, with parameterization assisted by efforts such as the Fire and Smoke Model Evaluation Experiment \cite{prichard2019fire}. Both stages inform the choice of two quantities of interest (QoIs): (i) a proxy of AOD for each fire and (ii) smoke footprint determination.

The chief bottleneck of our planned outer-loop analysis is the cost--accuracy trade-off of a smoke model. Quantifying power-output reduction requires sampling fire and smoke output over a range of fuel distributions. Coupled fire--atmosphere models such as WRF-Fire are the standard tool for landscape-scale plume modeling because they capture the two-way feedback between the fire and the atmospheric boundary layer, but this coupling incurs high computational cost, rendering Monte Carlo sampling infeasible. In contrast, faster spread models are either decoupled from the atmosphere or do not consider smoke transport, rendering Monte Carlo sampling inaccurate. We therefore construct a surrogate model, a significantly cheaper but sufficiently accurate approximation, that maps time-since-ignition to high-fidelity smoke concentration. 

Surrogate models fall into two main categories. One relies on internal (``intrusive'') simplifications of the high-fidelity model. An intrusive reduced-order model like Galerkin projection is impractical due to code complexity, and simplified physics like decoupled fire--atmosphere dynamics is insufficiently accurate for our use-case. Operational reduced-complexity smoke pipelines such as BlueSky \cite{larkin2009bluesky} adopt this physics-simplified route by linking modular emissions, plume-rise, and dispersion models, and are well-suited to short-term forecasting. The other category, which we adopt here, constructs a {\em data-driven} surrogate model trained on coupled fire--atmosphere output, retaining plume-scale accuracy across a wide envelope of fuel and weather scenarios. Existing data-driven smoke surrogates focus on bulk regional impacts~\cite{larsen2021deep, ko2013spatiotemporal, sun2023satellite}, short-term forecasting~\cite{lahrichi2025improved, huot2022next}, or remote sensing~\cite{larsen2021deep, ko2013spatiotemporal, sun2023satellite}. Additionally, \ifx\zacknote\undefined \citeA{larsen2022spatial} \else Larsen et al.~\cite{larsen2022spatial} \fi use kriging and Markov chain Monte Carlo to perform causal analysis on a spatially varying smoke field. Ongoing research is investigating full-field ML tools for near-term wildfire forecasting~\cite{shadrin2024wildfire, lahrichi2025improved} or ignition inversion~\cite{shaddy2024generative, hart2025real}. \ifx\zacknote\undefined \citeA{long2022comparing} \else Long et al.~\cite{long2022comparing} \fi study of the effect of fuel treatment on smoke output in Lake Tahoe using a state-of-the-art fuel succession model. Nominal values of PM$_{2.5}$ for each fuel type and fire stage yield smoke in a postprocessing step, but the underlying fire model lacks sufficient accuracy. In contrast to these methods, however, we must map between entire spatial fields, not just scalar values.

Our surrogate model, called a {\em multilinear operator}, learns a nonlinear operator with a linear least-squares problem~\cite{turnage2025optimal}. We wish to learn the fire-to-smoke operator $\mathscr{A}[f]=g$, where $f$ is the time-since-ignition field and $g$ is smoke concentration. First, we express $f$ and $g$ as truncated linear combinations:
\begin{equation}
f(x, y) \approx \sum_{n=1}^{r} a_n \phi_n(x, y), \qquad \qquad g(x,y) \approx \sum_{n=1}^{\tilde{r}} b_n \varphi_n(x, y) \, .
\label{eq:basis}
\end{equation}
Here, $\phi_n$ and $\varphi_n$ are basis functions for the set of time-since-ignition fields and smoke fields, and $a_n$ and $b_n$ are the encodings of the inputs and outputs, respectively. We will fit a polynomial $p$ that maps $(a_1, \dots a_{r}) \in \R^{r}$ to $( b_1, \dots, b_{\tilde{r}}) \in \R^{\tilde{r}}$ by solving a linear least-squares problem based on input--output observations. To perform inference from a new fire $\hat{f}$, we compute the expansion coefficients $\hat{a}_1, \dots, \hat{a}_{r}$, evaluate $p(\hat{a}_1, \dots, \hat{a}_{r}) = (\hat{b}_1, \dots, \hat{b}_{\tilde{r}})$, and reconstruct $g(x,y)$ with the basis expansion~\cref{eq:basis}. 

Multilinear operators exhibit many key advantages over traditional neural network approaches. Existing approaches like deep neural networks (DNN) or convolutional neural networks (CNN) are data-hungry and time-consuming to train. First, training takes at most a matter of seconds, as it requires solving a single matrix--vector equation. Inference is even faster, requiring a matrix--vector product, enabling massive ensembles. Second, if there is low-rank structure in the basis expansion \cref{eq:basis}, then $r$ and $\tilde{r}$ can be small, reducing the data requirements in the least-squares problem. This contrasts with pure linear regression on observational data, which is tethered to the spatial discretization and neglects higher-order interactions. Third, multilinear operators also come equipped with constructive, rather than existence-based, universal approximation guarantees, as well as robust theoretical bounds for errors induced by finite data collection~\cite{turnage2025optimal}.

The main contribution of this work is the implementation of a highly efficient operator-learning framework for mapping full-field fire observations to full-field smoke concentration. We characterize our model using 7,339 high-fidelity fire--smoke snapshots from coupled WRF-Fire/WRF-Chem simulations in the Upper Rio Grande Watershed, generated from LANDIS-II ignition events under three drought scenarios with checkpoints throughout each fire's burn period. We observe training times of at most 30 seconds (and often far lower) and inference times of a fraction of a second. We obtain a mean AOD field that is an order of magnitude more accurate than a Monte Carlo estimate with an equal number of high-fidelity observations. By applying a threshold to determine the presence of smoke at a spatial location, our method also obtains an area under the receiver operating characteristic curve (AUC) of 0.95 and intersection-over-union (IoU) of 0.65, while the most similar prior work obtains an IoU of 0.15.

We organize the rest of this paper as follows. \Cref{sec:background} covers background material on fire and smoke modeling. \Cref{sec:formulation} reviews data-driven operator learning and presents multilinear operators, including a new theoretical result. \Cref{sec:results} showcases the results of smoke prediction in the Upper Rio Grande Watershed. We conclude in \Cref{sec:conclusion}. Throughout, we include details necessary to ensure the trustworthiness and reproducibility of our study~\cite{jakeman2025verification} and specifically document them in~\cref{sec:appendix}.

%% file: background.tex
\section{Fire and Smoke Modeling} \label{sec:background}

Generating each input--output pair to train our surrogate involves three component models: fuel succession, fire spread, and smoke emission. Stochastic fuel-succession models capture the evolution of combustible fuels over time in response to natural phenomena or human intervention. The outputs of this stage are time-evolved fuel maps and randomized fire ignition points. Fire spread models propagate a fire according to fuel content, terrain, and atmospheric wind using the Rothermel spread equations~\cite{rothermel1972mathematical, albini1976computer}, coupled with a weather model. In tandem, a smoke emission model traces particulate matter through the atmosphere at each time step~\cite{kochanski2012wrf, kochanski2015toward, kochanski2019modeling}.

\subsection{Fuel Succession} \label{sssec:landis}

The current state of smoke forecasts is generally limited to the near-term, based on current vegetation states, fire behavior models, and weather forecasts \cite{mallia2023review}. The chemical and smoke transport component of these models range from simple statistical plume or box models to computationally expensive, dynamic, coupled models that allow for speciation of emissions. The selection of a particular model depends on topography, weather, spatial range, and the specific use-case. While the simpler models can be sufficient for very near term outlooks, they are limited in other contexts. For example, HYSPLIT, while computationally efficient through trajectory-based modeling, is incapable of capturing complex chemistry~\cite{stein2015noaa}. Operational constraints can be imposed to meet computational demands, e.g. static-fire assumptions in BlueSky~\cite{larkin2009bluesky} or longer timesteps in HRRR-Smoke~\cite{benjamin2016north}. 

This near-term nature of smoke models limits their application for longer-term forecasting, such as future wildfires with smoke occurring from fuel management operations \cite{long2022comparing}. Longer-term forecasting requires modeling fuel succession in a range of climate and weather scenarios, stochastic application of fire across the simulated landscape, and smoke spread from fires of different size, location, and biomass consumption. Fuel succession models such as LANDIS-II~\cite{scheller2007design, scheller2019landscape}, REBURN~\cite{prichard2023reburn, povak2023system}, and Envision~\cite{bolte2007modeling, spies2014examining, spies2017using} simulate forest ecosystem dynamics across landscape-level scales (1--$10^4$ km$^2$) and across decades to centuries. A variety of strategies are incorporated to simplify local changes due to environmental drivers. 

Fuel succession models, however, do not have an atmospheric transport component, limiting their utility in assessing the smoke impacts of future and management scenarios. In contrast, while fire propagation models such as WRF-Fire are capable of spreading smoke during a fire, their computational expense and lack of an internalized succession model reduces their capacity to address this challenge. 

The LANDIS-II forest landscape model simulates forest change given species-specific parameters governing seedling establishment, growth rates, and mortality. LANDIS-II is similar to other ecosystem models, e.g. \cite{bugmann2022evolution}, but possesses additional controls over species parameters with the PnET (Photosynthesis and Evapo-Transpiration) extension~\cite{gustafson2023pnet}. Likewise, the Dynamic Fuels and Fire System (DFFS) extension simulates fire and fuel interactions on these landscapes and incorporates their effects into the annualized landscape outputs \cite{sturtevant2009simulating}. The DFFS uses simple Huygens-based algorithms, similarly to FARSITE, for fire spread, which are dependent on fuel type, weather, and topography \cite{sturtevant2009simulating}. Additionally, LANDIS-II can simulate different fuel-management scenarios, including the effects of fuel thinning and prescribed burning at various levels of intensity. These simulations provide grid-cell-level fuel types and ignition points for each year. These fuel layers, as well as internal model details are described in more detail in \cite{remy2024restoring, scheller2007design}.  

\subsection{Fire Propagation} \label{sssec:wrf-fire}

Although a comprehensive review of fire-propagation models is beyond the scope of this paper, we offer a brief account of their historical development; see \cite{bakhshaii2019review} for a discussion of next-generation models. One of the first full-fledged models is BEHAVE~\cite{andrews1986behave, andrews2013current}, which is point-based and local. Later models incorporated spatial variability of the fire. FARSITE, later FlamMap, employs Huygens' principle for vector propagation in a two-dimensional plane but is decoupled from atmospheric processes. Neglecting feedback between the fire and the atmosphere lowers model accuracy. Like FlamMap, CAWFE uses vector propagation, but generalizes from Huygens' principle and couples the fire and atmosphere via the Clark--Hall weather model~\cite{clark1996coupled, clark1996coupled2, coen2013modeling, clark2004description, clark1977small, clark1991multi}. In contrast, HIGRAD/FIRETEC utilizes a coupled system of partial differential equations (PDEs) describing temperature, fuel content, atmospheric winds, and chemical species~\cite{linn1997transport, linn2002studying, clark2010subgrid, dupuy2011exploring, reisner2000coupled}. For in-the-field use, QUIC-fire~\cite{linn2020quicfire} simplifies FIRETEC with an approximation of both the wind and fire-spread models. Lastly, level-set methods for wildfires use a PDE to track the fire front~\cite{mallet2009modeling, osher1988fronts}, bypassing vector propagation. Using the Weather Research and Forecasting (WRF) atmospheric model~\cite{skamarock2019wrf}, WRF-Fire is a fully coupled fire--atmosphere model based on level sets~\cite{mandel2011coupled, coen2013wrf, munozesparza2018accurate}. WRF-Fire largely superseded CAWFE and is the simulation code used in this paper.

The WRF backend is a mesoscale numerical weather prediction model developed by the National Center for Atmospheric Research to provide open-source real-time forecasting, research, and analysis across spatial and temporal scales~\cite{skamarock2019wrf}. WRF-Fire couples a fire to the atmosphere through the lowest atmospheric level, using wind, topography, and fuel to compute the outward-normal spread rate $S$ via Rothermel's equations~\cite{rothermel1972mathematical, albini1976computer}. As the fire spreads, the ignited fuels burn until complete combustion. WRF-Fire uses Anderson's 13 fuel categories~\cite{anderson1982aids} by default, but we use Scott and Burgan's 40 fuel categories~\cite{scott2005standard} for consistency with LANDIS, as in \cite{decastro2022computationally}. The resolution of the fire grid within WRF-Fire is limited by the spatial scales of topographic and fuel inputs (here, LANDIS-II fuels).

WRF-Fire has two component models for the fire and atmosphere. The level-set equation describes the propagation of the fire front:
\begin{equation}
    \frac{\partial \psi}{\partial t} + S(\vec{u}, \vec{x}, t)\| \nabla \psi\|_2 = 0, 
    \label{equ:lfn}
\end{equation}
where $\psi(\vec{x}, t)$ is the level-set function, parameterized as the signed distance to the fire front~\cite{mallet2009modeling, osher1988fronts}. The time-since-ignition, i.e. the input to the surrogate model, is given by
\begin{equation}
    f(x, y, t) = \begin{cases} 
    t - \sup \{ t_0 : \psi(x,y,t_0) \geq 0 \}, \qquad & \text{if}~\psi(x,y,t) < 0 \\
    0, & \text{otherwise}
    \end{cases} \ .
\end{equation}
The Rothermel spread rate $S$ is based on fuel (moisture, density, ratio of surface area to volume), topography, and winds $\vec{u} = (u, v, w)$. In turn, the heat output of the fire appears as a forcing term in the WRF atmospheric model for $\vec{u}$, yielding two-way coupling:
\begin{subequations}
    \label{equ:euler}
    \begin{align}
        \frac{\partial \rho}{\partial t} + \vec{u} \cdot \nabla \rho &= 0 \\
        \frac{\partial \vec{u}}{\partial t} + \vec{u} \cdot \nabla \vec{u} &= -\frac{\nabla p}{\rho} + \nabla V \\
        \frac{\partial T}{\partial t} + \nabla \cdot(\vec{u} T)&= f(\psi)\\
        \frac{\partial V}{\partial t} + \vec{u} \cdot \nabla V &= gw
    \end{align}
\end{subequations}
In~\cref{equ:euler}, $\rho$ is the density of air, $T$ is temperature, $V$ is the geopotential height, and $p$ is pressure. Pressure and temperature are related via
$$
p = p_0 \left( \frac{\rho R T}{p_0} \right)^{c_p/c_v} 
$$
where $p_0 = 10^5$ Pa is the surface pressure, $R$ is the specific gas constant for dry air, $c_p = 1005$ J/(kg$\cdot$K) is the specific heat of dry air at constant pressure, and $c_v = 718$ J/(kg$\cdot$K) is the specific heat of dry air at constant volume. Equation \Cref{equ:lfn} is discretized in space with the finite-volume method and in time with the explicit second-order Runge--Kutta method~\cite{mandel2011coupled}. Equation \Cref{equ:euler} is discretized in space with finite differences~\cite{knievel2007explicit} and in time with the explicit third-order Runge--Kutta method~\cite{mandel2011coupled}. Operator splitting maintains the coupling between \crefrange{equ:lfn}{equ:euler}. Boundary conditions come from either idealized Dirichlet conditions or real weather data; see \Cref{sec:setup}.

\subsection{Smoke Transport} \label{sec:wrf-chem}

Additional modules have allowed WRF-Fire to inject passive smoke tracers to the atmosphere \cite{kochanski2012wrf, kochanski2015toward, kochanski2019modeling, anderson2004fire}. WRF-Chem models the transport of chemical species ($c$) in the atmosphere via the convection--diffusion--reaction equation
\begin{equation}
\frac{\partial c}{\partial t} - \kappa \nabla^2 c + \nabla \cdot (\vec{u} c) + \gamma(c) = 0 \, , \qquad \vec{x} \in \Omega \subset \R^3,~t>0,
\label{equ:smoke_conc}
\end{equation}
where $\gamma$ reflects reactions among species, the atmosphere, and each other. Smoke amounts per timestep are a function of the pixel-wise burn fractions and fuel density, as well as an empirical factor representing the percentage of burned fuel emitted as smoke, provided as 0.02 by default in WRF-Chem~\cite{grell2005fully}.

We require smoke concentration $c$ to have units of $\mu$g/m$^3$, but WRF-Chem reports the mass mixing ratio of smoke in units of \textit{g smoke/kg dry air}. Using the ideal gas law, we convert to concentration via
\begin{equation}
    c_{\text{g/m}^3} = c_{\text{g/kg}} \left(  \frac{p}{R \ T_{\text{a}}} \right)
\end{equation}
where $R = 287$ J/(kg$\cdot$K) is the specific gas constant of dry air, and $T_{\text{a}}$ is the actual, not potential, temperature:
$$
T_{\text a} = T \left( \frac{p}{p_0} \right)^{R/c_p} \ .
$$
As the variable smoke represents the sum total of all possible emissions from a burned fuel source, we can assume that the plume shape and concentration is representative of any emission from the fire. 

We project $c$ onto the $(x, y)$ plane as the cumulative mean value over the smoke column to define the output of the surrogate model:
\begin{equation}
    \label{equ:smoke}
    g(x, y, t_k) = \sum_{j=0}^{k} \ \frac{1}{z_{\text{top}}-z_0} \int_{z_0}^{z_{\text{top}}} c(x, y, z, t_j)\, \d z
\end{equation}
where $c(x, y, z, s)$ is the smoke concentration at time $s$ in $\mu\text{g} / \text{m}^3$.  Since \cref{equ:smoke_conc} is spatially discretized with the finite-volume method, we can straightforwardly approximate~\cref{equ:smoke} using the second-order composite midpoint rule. Importantly, we define \cref{equ:smoke} so that $g$, like time-since-ignition, is non-decreasing in time:
$$
g(x, y, t_1) \leq g(x, y, t_2) \qquad \text{for all } x, y\in \R~\text{and}~t_1 \leq t_2 \, ,
$$
Thus, each ultimate smoke plume $g$ represents both the complete spatial extent of the plume up to that point and locations with the highest smoke concentrations. This allows each time step $g$ to describe the smoke plume's spatial extent up until that time step. This monotonicity avoids convection-dominated dynamics, which is crucial for an accurate basis encoding later (\Cref{ssec:pca}).

%% file: formulation.tex
\section{Operator Learning} \label{sec:formulation}

Mapping one spatial field to another requires operator learning, which can be broadly categorized into four main categories: dynamical systems, integral operators, unknown basis functions, and known basis functions. Many of the earliest contemporary works on operator learning are concerned with dynamical systems.  Operator Inference first projects a system onto a reduced basis with proper orthogonal decomposition (POD) and then learns a time-propagation map via a least-squares fit~\cite{peherstorfer2016data, qian2020lift}. Flow Map Learning~\cite{churchill2023flow} parametrizes the time-evolution map as a feedforward network, possibly in a basis representation~\cite{churchill2025principal}, while \ifx\zacknote\undefined \citeA{patel2018nonlinear} were \else \cite{patel2018nonlinear} was \fi the first to approximate a nonlinear time-evolution operator directly in Fourier space. Another approach here is to learn a nonlinear input-to-observables map and then learn the linear Koopman operator~\cite{yeung2017learning}. 

The latter three categories, however, are applicable to general input--output maps. Many methods treat the layers of a network as a composition of local and integral operators passed through nonlinear activation functions. Graph Neural Operators treat the inputs to a convolution as the edge features of a graph~\cite{li2020neural}. Fourier Neural Operators~\cite{li2021fourier} exploit the Fourier convolution theorem to learn convolution operators in Fourier space, possibly with low-rank tensor decompositions~\cite{kossaifi2024multigrid}. More generally, Kernel Neural Operators parametrize a kernel as a neural network and do not apply forward or inverse transforms repeatedly~\cite{lowery2025kernel}. To learn unknown basis functions, Deep Operator Networks (DeepONets) utilize two separate networks to learn a function basis for the output space (trunk network) and a set of combination coefficients (branch network)~\cite{lu2021learning}.  The last category of methods assumes a known linear basis for the input and output spaces, which enables an offline transform before learning a discretization-invariant mapping. A variant of DeepONet, using a PCA representation for the trunk, falls into this category~\cite{lu2022comprehensive}. Motivated by aliasing errors arising from repeated forward/inverse transforms, Spectral Neural Operators~\cite{fanaskov2023spectral} use a standard feedforward network to learn the coefficient-to-coefficient map. 

Multilinear operators~\cite{turnage2025optimal}, which are the focus of this paper, learn a polynomial map between the input and output coefficients with a least-squares solve, which is orders of magnitude faster than training deep neural network models. Furthermore, multilinear operators come equipped with theoretical guarantees on data collection, as well as {\em constructive}, rather than existence-based, universal approximation theorems. 

This section will first develop the general formulation of data-driven operator learning, followed by multilinear operators. We will then discuss our choice of $\phi$ and $\varphi$ for the encoding/decoding step. With practical simplifications for the use-case of smoke modeling, we can derive a closed-form expression for the approximation weights with a linear operator. We then discuss higher-order operators and end with a brief discussion of computational cost.

\subsection{Data-Driven Operator Learning}

Let $\op: \opinputspace \to \opoutputspace$ be an operator, i.e. $\op [\opinput] = \opoutput$ where $f$ and $g$ are functions \hl{representing inputs and outputs of interest}. Data-driven operator learning approximates $\op$ using input--output pairs $\{ (\opinput^{(m)}, \opoutput^{(m)}) \}_{m=1}^M$. For our setting, we use the weighted function spaces $\opinputspace = L^2_\eta(\Omega; \R)$ and $\opoutputspace = L^2_\nu(\Omega; \R)$, where $\eta$ and $\nu$ are probability measures and $\Omega \subset \R^2$ is compact. Since $\opinputspace$ and $\opoutputspace$ are separable Hilbert spaces, there exist orthonormal bases $\{ \basisinput_n(\vec{x}) \}_{n=1}^\infty$ for $\opinputspace$ and $\{ \basisoutput_n(\vec{y}) \}_{n=1}^\infty$ for $\opoutputspace$ \hl{that represent the inputs and outputs}. These bases, \hl{which are in general not the same}, have the properties
\begin{alignat*}{3}
    \ip{\basisinput_i}{\basisinput_j}_{\opinputspace} &= \delta_{ij}, \qquad &&  
    \opinput(\vec{x}) &&= \sum_{n=1}^\infty \ip{\opinput}{\basisinput_n}_\opinputspace \, \basisinput_n(\vec{x})~~\text{almost everywhere}, \\
    \ip{\basisoutput_i}{\basisoutput_j}_{\opoutputspace} &= \delta_{ij}, \qquad && \opoutput(\vec{y}) &&= \sum_{n=1}^\infty \ip{\opoutput}{\basisoutput_n}_\opoutputspace \, \basisoutput_n(\vec{y})~~\text{almost everywhere}, 
\end{alignat*}
for all $\opinput \in \opinputspace$ and $\opoutput \in \opoutputspace$. For $u, v \in \opinputspace$, the inner product $\ip{u}{v}_\opinputspace$ is given by
$$
\ip{u}{v}_{\opinputspace} \coloneqq \int_{\Omega} u(\vec{x}) v(\vec{x}) \, \d \eta(\vec{x}) \, ,
$$
and similarly for the inner product on $\opoutputspace$. After truncating the inputs and outputs at index $r$ and $\tilde{r}$ respectively, the crux of operator learning is to approximate the {\em finite-dimensional} map
$$
(\ip{f}{\phi_1}_\opinputspace, \dots, \ip{f}{\phi_r}_\opinputspace) \qquad \mapsto \qquad (\ip{g}{\varphi_1}_\opoutputspace, \dots, \ip{g}{\varphi_{\tilde{r}}}_\opoutputspace)
$$
between input and output coefficients.

\subsection{Multilinear Operators}

To learn a suitable approximation $\hat{\op}$ of $\op$, we must first define a class of candidate operators. Finite-dimensional polynomial regression approximates a function using the basis $\{ P_k(x) \}_{k=0}^\infty$, where $P_k(x)$ is a polynomial of degree $k$. Analogously, we can express $\op$ as a linear combination of polynomial basis elements:
\begin{equation}
\hat{\op}_{\vec{\theta}}[f] = \sum_{i=1}^{\tilde{r}} \ \ \sum_{\vec{j} \in \Lambda \subset [d]_0^{r}} \theta_{i,\vec{j}} \ \underbrace{\basisoutput_i \prod_{k=1}^{r} P_{j_k} (\ip{\basisinput_{k}}{\opinput})}_{\Phi_{i, \vec{j}}[\opinput]} \, .
\label{equ:approximate_operator}
\end{equation}
We denote the candidate class of such operators $\surrclass = \{ \hat{\op}_{\vec{\theta}} : \vec{\theta} \in \N_0 \times \Lambda \}$. In~\cref{equ:approximate_operator}, $\theta_{i, \vec{j}}$ are the trainable parameters, $\vec{j}$ is a vector of natural numbers, $\Lambda$ is the set of admissible indices,\footnote{$\Lambda$ must also be lower-complete, i.e. for each $\vec{i} \in \Lambda$, $\{ \vec{j} \in \N_0^r: j_k \leq i_k~\text{for all}~1 \leq k \leq r\} \subset \Lambda$.} $d$ is the maximal polynomial degree, and $[d]_0 = \{0, 1, \dots, d\}$. Importantly, even though \cref{equ:approximate_operator} can be nonlinear in $f$, it is {\em linear} in $\vec{\theta}$, namely a linear combination of rank-one operators $\Phi_{i, \vec{j}}$.

With a suitable inner product on the candidate space, we can solve for $\vec{\theta}$ with least-squares {\em very efficiently.} Let $\opspace = L^2_\rho(\opinputspace; \opoutputspace)$ with the inner product
$$
\ip{\op}{\mathscr{B}}_\opspace = \int_\opinputspace \ip{\op[\opinput]}{\mathscr{B}[\opinput]}_\opoutputspace \ \d \rho (\opinput), \qquad \| \op \|^2_\opspace = \ip{\op}{\op}_\opspace \, .
$$
We want to choose $\theta_{i, \vec{j}}$ to solve
\begin{equation}
\min_{\vec{\theta}} \| \op - \hat{\op}_{\bm \theta} \|_\opspace = \min_{\hat{\op} \in \surrclass} \| \op - \hat{\op} \|_{\opspace} \, .
\label{equ:lstsq}
\end{equation}
Fortunately, \cref{equ:lstsq} has a known solution. To keep the notation simple, we introduce a lexicographic index for $\vec{\theta}$ and $\Phi$, i.e. one that unravels the multi-index notation. The optimal $\vec{\theta}$ solves the linear system
\begin{equation}
    \vec{G} \vec{\theta} = \vec{y}, \qquad \vec{G}_{ij} = \ip{\Phi_i}{\Phi_j}_\opspace, \qquad \vec{y}_i = \ip{\Phi_i}{\op}_\opspace \, ,
    \label{equ:gramian}
\end{equation}
where $\vec{G}$ is the Gram matrix. Since the size of $\vec{G}$ grow linearly in the basis size, and the basis grows quickly with increasing degree $d$ unless $\Lambda$ is sparse, we limit ourselves to $d \leq 2$.

\subsection{Basis Selection} \label{ssec:pca}

There are many possible choices for $\phi$ and $\varphi$ above, including orthogonal polynomials~\cite{sharma2026polynomial, herrmann2024neural, westermann2026performance}, a Fourier basis~\cite{li2021fourier, kossaifi2024multigrid}, or eigenfunctions of a known Mercer kernel~\cite{lowery2025kernel}. Here we briefly overview principal component analysis~\cite{hotelling1933analysis, shlens2014tutorial}, which gives an empirical approximation of the eigenfunctions of an unknown covariance kernel from input--output data \hl{corresponding to specific simulation variables $f$ and $g$}. For simplicity, we present the snapshot formulation rather than the infinite-dimensional formulation.

Denote column-ordered input and output snapshot matrices as
$$
\vec{S}_{ij} = f^{(j)}(x_i), \qquad \vec{\tilde S}_{ij} = g^{(j)}(y_i)
$$
where $\vec{S} \in \R^{N \times M}$ and $\vec{\tilde S} \in \R^{\tilde{N} \times M}$ with $M > \max \{ N, \tilde{N} \}$. Without loss of generality, assume the snapshots are mean-centered. From the singular value decomposition (SVD) of $\vec{S}$ and $\vec{\tilde{S}}$
\begin{equation}
\vec{S} = \vec{U} \vec{\Sigma} \vec{V}^\top, \qquad \vec{\tilde S} = \vec{\tilde U} \vec{\tilde \Sigma} \vec{\tilde V}^\top \, ,
\label{equ:svd}
\end{equation}
we obtain the principal components (columns of $\vec{U}$ and $\vec{\tilde{U}}$) of the snapshots, along with singular values $\sigma_1 \geq \cdots \geq \sigma_N$ that measure the components' influence. The matrix $\vec{U}$ is $N \times N$ orthogonal, $\vec{\Sigma} = \text{diag}(\sigma_1, \dots, \sigma_N)$, and $\vec{V}$ is $M \times N$ with orthonormal columns (and similarly for $\vec{\tilde{S}}$). 

By choosing some truncation levels $r \ll N$ and $\tilde{r} \ll \tilde{N}$ for which $\sigma_r > 0$ and $\tilde{\sigma}_{\tilde{r}} > 0$, we get the proper orthogonal decomposition:
\begin{equation}
\begin{alignedat}{3}
\vec{\hat{f}} &= \vec{U}_r^\top \vec{f} \in \R^r, \qquad \vec{S}_r &&= \vec{\Sigma}_r \vec{V}_r^\top \in \R^{r \times M}\, , \\
\vec{\hat{g}} &= \vec{\tilde{U}}_{\tilde{r}}^\top \vec{g} \in \R^{\tilde{r}}, \qquad \vec{\tilde{S}}_{\tilde{r}} &&= \vec{\tilde \Sigma}_{\tilde{r}} \vec{\tilde V}_{\tilde{r}}^\top\, \in \R^{\tilde{r} \times M} \, .
\end{alignedat}
\label{equ:pca_basis}
\end{equation}
The reduced matrix $\vec{U}_r$ is the submatrix of the first $r$ columns of $\vec{U}$, and $\vec{\Sigma}_r = (\sigma_1, \dots, \sigma_r)$.

\subsection{Practical Simplifications}
In general, computing $\ip{\Phi_i}{\Phi_j}_\opspace$ is not tractable. We estimate the inner product by computing the sample expectation 
$$
\ip{\Phi_i}{\Phi_j}_\opspace \approx \ip{\Phi_i}{\Phi_j}_{\opspace, M} \coloneqq \frac{1}{M} \sum_{m=1}^M \ip{\Phi_i[\opinput^{(m)}]}{\Phi_j[\opinput^{(m)}]}_\opoutputspace \ .
$$
We can also leverage the physical properties of a fire-to-smoke map to simplify the candidate class~\cref{equ:approximate_operator}. First, we know that $\op$ is non-affine because an empty fire never results in positive smoke, so we restrict $P_j$ to $j \geq 1$. Second, fire sizes are sampled from an unknown probability distribution $\d \rho(f)$. It is desirable for $P_j$ to be a specific orthogonal polynomial family based on $\rho$~\cite{turnage2025optimal}, so that $\vec{G} = \vec{I}$ in~\cref{equ:gramian}. However, since $\rho$ is not known in closed form, we use a low-degree monomial basis $P_j(x) = x^j$ at the trade-off of a larger condition number in~\cref{equ:gramian}. Third, in keeping with the PCA basis, we use an unweighted inner product for $\opinputspace$ and $\opoutputspace$. We interpret the reduced basis vectors $\vec{u}_n$ and $\vec{\tilde{u}}_n$ as discretizations of orthogonal basis functions $\{ \phi_n(x)\}_{n=0}^r$ and $\{ \varphi_n(y) \}_{n=0}^{\tilde{r}}$, where the $L^2$ inner product over over $\opinputspace$ and $\opoutputspace$ is approximated with the Euclidean dot product.

\subsection{Linear Case}
For a linear operator, we can make even further simplifications. Here, $\Lambda$ in~\cref{equ:approximate_operator} contains the purely linear elements, omitting the constant and quadratic terms:
$$
\Lambda = \{ 0, 1 \}^{\tilde{r}} \ \backslash \ (\{ 0, \dots, 0 \} \cup \{ 1, \dots, 1\}) \, .
$$
Then we get
$$
\hat{\op}_{\vec{\theta}}[f] = \sum_{i=1}^r \sum_{j=1}^{\tilde{r}} \theta_{i,j} \ \basisoutput_i \ip{\basisinput_{j}}{\opinput} \, ,
$$
which has a closed-form solution given by the following theorem.

\begin{theorem}
    The solution of \cref{equ:lstsq} over a linear candidate space with the PCA basis \cref{equ:pca_basis} is
    \begin{equation} 
        \vec{\theta} = \vec{\tilde{\Sigma}}_{\tilde{r}} \vec{\tilde V}_{\tilde{r}}^\top \vec{V}_r \vec{\Sigma}_r^{-1} \, .
        \label{equ:lstsq_pca_linear}
    \end{equation} 
    \label{thm:lstsq_pca_linear}
\end{theorem}

\begin{proof}
Using the original double-indexing on $\Phi$, the Gram matrix elements become
\begin{align*}
\ip{\Phi_{i,j}}{\Phi_{k, \ell}} &\approx \frac{1}{M} \sum_{m=1}^M (\vec{\tilde u}_i^\top \vec{\tilde u}_k) (\vec{u}_j^\top \vec{f}^{(m)}) (\vec{u}_\ell^\top \vec{f}^{(m)}) \\
&= \delta_{ik} \frac{1}{M} \sum_{m=1}^M \sigma_j (\vec{V}^\top_r)_{jm} \sigma_{\ell} (\vec{V}^\top_r)_{\ell m} \\
&= \delta_{ik} \frac{1}{M} \sum_{m=1}^M \sigma_j (\vec{V}^\top_r)_{jm} (\vec{V}_r)_{m \ell} \sigma_\ell \\
&= \delta_{ik} \frac{1}{M} (\vec{\Sigma}_r \vec{V}_r^\top \vec{V}_r \vec{\Sigma}_r)_{j\ell}  \\
&= \frac{1}{M} \delta_{ik}(\vec{\Sigma}_r^2)_{j\ell} \ .
\end{align*}
The right-hand side becomes
\begin{align*}
    \ip{\Phi_{i,j}}{\op} &\approx \frac{1}{M} \sum_{m=1}^M (\vec{\tilde{u}}_i^\top \vec{g}^{(m)}) ( \vec{u}_j^\top \vec{f}^{(m)} ) \\
    &= \frac{1}{M} \sum_{m=1}^M \tilde{\sigma}_i (\vec{\tilde{V}}_r^\top)_{im} \sigma_j (\vec{V}^\top_r)_{jm} \\
    &= \frac{1}{M} \sum_{m=1}^M \tilde{\sigma}_i (\vec{\tilde{V}}^\top_r)_{im} (\vec{V}_r)_{mj} \sigma_j  \\
    &= \frac{1}{M} (\vec{\tilde \Sigma}_r \vec{\tilde V}_r^\top \vec{V}_r \vec{\Sigma}_r)_{ij} \, .
\end{align*}
So the coefficients $\vec{\theta} = (\theta_{i,j}) \in \R^{\tilde{r} \times r}$ solve the matrix equation
$$
\vec{\theta}\vec{\Sigma}^2_r  = \vec{\tilde{\Sigma}}_{\tilde{r}} \vec{\tilde V}_{\tilde{r}}^\top \vec{V}_r \vec{\Sigma}_r,
$$
from which the result follows.
\end{proof}

\subsection{Quadratic Case}

In the nonlinear case, there is not a closed-form expression like~\cref{equ:lstsq_pca_linear}. However, there is considerable ease in constructing the Gram matrix $\vec{G}$ and right-hand side $\vec{y}$.

Denote by $\vec{A} = \vec{V}_r \vec{\Sigma}_r \in \R^{M \times r}$ the transpose of the reduced-dimensional input snapshot matrix. Let $\vec{M}$ be the matrix containing the linear and quadratic interaction terms of $\vec{A}$. So
$$
\vec{M} = \begin{bmatrix} \vec{a}_1, \cdots, \vec{a}_r \ \big| \ \vec{a}_1\odot \vec{a}_1, \, \cdots, \vec{a}_1 \odot \vec{a}_r, \ \ \vec{a}_2 \odot \vec{a}_2, \, \cdots,  \vec{a}_2 \odot \vec{a}_r, \ \cdots, \ \vec{a}_{r} \odot \vec{a}_r \end{bmatrix} \, ,
$$
where $\vec{a}_k$ is the $k$-th column of $\vec{A}$ and $\odot$ denotes the Hadamard (componentwise) product. Then the coefficients $\bm \theta$ solve the matrix equation
$$
\vec{M}^\top \vec{M} \vec{\theta}^\top = \vec{M}^\top \vec{\tilde V}_{\tilde{r}} \vec{\tilde \Sigma}_{\tilde{r}} \, . 
$$
Now, $\vec{M}$ has $r + r(r+1)/2$ columns, so if $M \ll r^2$, then $\vec{M}^\top \vec{M}$ is singular. One possible mitigation approach is to incorporate Tikhonov regularization with a parameter $\lambda > 0$. Then the matrix equation becomes
\begin{equation}
( \vec{M}^\top \vec{M} + \lambda \vec{I}) \vec{\theta}^\top = \vec{M}^\top \vec{\tilde V}_{\tilde{r}} \vec{\tilde \Sigma}_{\tilde{r}} \, , 
\label{equ:lstsq_pca_quadratic}
\end{equation} 
and the optimal choice of $\lambda$ can be tuned as a hyperparameter. One can interpret $\lambda$ either as a penalty term on large values of $\vec{\theta}$ or, in the Bayesian view, as arising from an i.i.d. Gaussian prior distribution on $\vec{\theta}$ with mean zero and variance $1/\lambda$. 

\subsection{Cost Analysis}

Using the PCA representation, $\vec{\theta}$ in the linear case comes from three matrix multiplications, two of which are diagonal. The cost of computing~\cref{equ:lstsq_pca_linear} is $O(2r\tilde{r} + Mr \tilde{r})$. As a result, the cost of computing $\vec{\theta}$ is negligible after two SVDs.

The quadratic case, however, is noticeably more expensive than the linear case. The cost of solving~\cref{equ:lstsq_pca_quadratic} involves computing $\vec{M}^\top \vec{M}$ and $\vec{M}^\top \vec{\tilde{V}}_{\tilde{r}}$, which is $O(Mr^4)$ work; solving a linear system of size $r + r(r+1)/2$, which is $O(r^6)$ work; and applying the solution to each of the $\tilde{r}$ columns of the right-hand side, which is $O(\tilde{r} r^4)$ work. However, for any reasonable $r$ and $\tilde{r}$, with modern linear algebra software, this cost is still quite small relative to the cost of training a large neural network.

%% file: results.tex
\section{Computational Results and Discussion} \label{sec:results}

Here we use multilinear operators, both linear and quadratic, to approximate smoke concentration fields in the Upper Rio Grande Watershed~\cite{remy2024restoring} and to determine areas where smoke crosses a critical threshold, a task not previously attempted in the literature. Importantly, we are only verifying our method against a trusted emulator (WRF-Fire) and not calibrating a model probabilistically against real observations~\cite{jakeman2025verification}. Slight differences in plume location or plume height (including spurious plumes) will have an outsize effect on $\ell^2$ error for our use-case. As a result, we consider classification accuracy (\Cref{sec:classification}) primarily and least-squares error \cref{equ:lstsq} secondarily. All reported runtimes are on a 2023 Macbook Pro with 32GB RAM and an Apple M2 Max processor.

The first step of any analysis, whether utilizing a surrogate model or not, is to generate high-fidelity data (\Cref{sec:setup}). In \Cref{sec:monte-carlo}, we establish that building a surrogate model is actually worth the effort, by demonstrating reduced computational cost for a given accuracy threshold in a proxy for mean aerosol optical depth. In \cref{sec:classification}, we use our method to identify specific areas where smoke is present, attaining compelling classification accuracy measured by area-under-the-curve and intersection-over-union. In \Cref{sec:hyperparameter-sensitivity}, we show that, for the classification QoI, our method is mostly insensitive to the particular basis truncation and smoke thresholds except at the extremes, i.e. including almost all smoke or almost all variance. We quantify the performance of our method when the training and testing distributions are different in \cref{sec:distributional-shift}. We consider an alternative method of learning a coefficient-to-coefficient map in \cref{sec:gp}. We discuss implications, limitations, possible extensions, and other baselines in \cref{sec:discussion}.

\subsection{High-Fidelity Data Generation}
\label{sec:setup}

\subsubsection{Fuel Succession} 
LANDIS-II simulations from \cite{remy2024restoring}, \hl{spanning 20--50 years into the future}, provide the fuel distribution and ignition point that WRF-Fire takes as inputs. For mountainous areas within the Upper Rio Grande Watershed, we translate the LANDIS-II fuels to the 40 Scott and Burgan categories~\cite{scott2005standard}; to fill lower elevation gaps, we use fuel layers from the publicly available LANDFIRE dataset~\cite{landfire2022}. Each LANDIS-provided ignition point is run with three sets of weather conditions in WRF-Fire (low, medium, and high drought). We use the U.S. Drought Monitor for the period of 2010--2025 to identify three summer periods in which the extreme drought condition was under 25\%, between 50\% and 75\%, and greater than 75\% of the Western United States. 

\subsubsection{WRF-Fire}

As inputs, WRF-Fire ingests fuels and ignition points from LANDIS-II along with weather data, and provides burn area and smoke as outputs. For each \hl{LANDIS year}, we specify the boundary conditions of~\cref{equ:euler} using weather sets from the European Center for Medium-Range Weather Forecasts' fifth-generation atmospheric reanalysis of the global climate~\cite{hersbach2020era5} at approximately 30 km grid-spacing. We additionally alter the surface-fuel moisture percentage within WRF-Fire for each weather set: 0\%, 5\%, and 10\% for the high, medium, and low drought scenarios, respectively. We use the WRF Preprocessing System~\cite{wps46} to generate input files for WRF-Fire from weather datasets. 

For the numerical discretization, we increase resolution in WRF-Fire by a factor of 10 through two nested domains, first to a 3 km inner domain, then to the fire grid with 300 m spacing, matching the resolution of the LANDIS-II input fuel layers. The outer and inner domains have 58 vertical layers with a top height of 20 km above sea level. The inner atmospheric domain has size $100 \times 80$, and fire grids are $1000 \times 800$. Fires burn at the fire grid resolution, while plumes are generated at the inner domain resolution. To maintain the CFL condition, we use $\Delta t = 120$s timestep for the outer domain $\Delta t = 40$s the inner domain and fire grid. Lastly, WRF-Chem propagates the smoke concentration according to \Cref{sec:wrf-chem}. Checkpoints are made every six hours of simulation time. \hl{Thus, the temporal resolution of our fire and smoke data is six hours, although the fuel succession model simulates between 20 and 50 years.}

\subsubsection{Datasets}
Since our focus is on large wildfires, we first filter out all burn areas smaller than $10^5$ acres, a heuristic cutoff slightly smaller than Albuquerque, NM. This leaves $M = 7339$ WRF-Fire snapshots, each of size $300~\text{km} \times 240~\text{km}$ with $\Delta x=\Delta y=3$ km. On average, each high-fidelity snapshot takes 684 seconds on a high-performance computing cluster using 20 cores of an Intel Sapphire Rapids processor with 256GB memory; although separate fires can be run in parallel, fire generation takes nearly 14,000 core-hours. We adopt a 45/10/45 split for training, validation, and test datasets. We use the validation set to tune hyperparameters (\Cref{sec:hyperparameter-sensitivity}) and document PCA generalization error. To prevent data leakage, all snapshots from a given fire are in only one dataset. We include fires from all drought conditions in the training set, except to study distributional shift in \cref{sec:distributional-shift}. 

\hl{We select the inputs so that the many-query online stage requires only an inexpensive fire simulator, such as LANDIS-II or FARSITE, instead of an expensive coupled code like WRF-Fire.} The inputs $f^{(m)}$ are time-since-ignition profiles; the outputs $g^{(m)}$ are average smoke concentrations over the vertical air column. Topography, weather, and other WRF inputs appear indirectly in the time-since-ignition and smoke concentration profiles but are {\em not} directly supplied as training data. \hl{We re-emphasize that our proposed surrogate model can be applied in principle to any desired input or output variables.}

\subsection{QoI 1: Aerosol Optical Depth}
\label{sec:monte-carlo}

Aerosol optical depth quantifies the haziness of an air column. AOD is defined as the logarithm of the ratio of received radiation to transmitted radiation. The precise relationship between smoke and transmitted radiation is beyond the scope of the present work, so we use the following QoI as a proxy:
\begin{equation}
    Q(x,y) = \mathbb{E}_g\left[ \ln(g(x,y)+1) \right] \, .
    \label{equ:qoi}
\end{equation}
We take the expectation over all {\em final-time} cumulative smoke outputs $g$ in the holdout set. Even though $g$ represents cumulative rather than instantaneous smoke, this proxy is still an informative metric of haziness over a long time-window. The inner argument, $\ln(g+1)$, captures perfect transmission in the absence of smoke ($0$ when $g \equiv 0$) and increases as $g$ increases.

To approximate $Q$, we consider three approaches. In the first, we use Monte Carlo sampling of the full state:
$$
\overline{Q}_M(x,y) = \frac{1}{M}\sum_{m=1}^M \ln (g^{(m)}(x,y)+1)
$$
In the second, we directly learn the map $f \mapsto \mathscr{A}_{\vec{\theta}}[f] = \ln(g+1)$ with a multilinear operator. As part of this process, we apply PCA to the fire snapshots $\{f^{(m)} \}$ and QoI snapshots $\{ \ln(g^{(m)}+1) \}$, to get reduced basis matrices $\vec{U}_r \in \R^{N \times r}$ and $\vec{\tilde{U}}_{\tilde{r}} \in \R^{N \times \tilde{r}}$ respectively. As we will soon show, multilinear operators are highly inexpensive to train, so it is perfectly feasible to build a surrogate tailored to a particular QoI. The approximation of $Q$ in this case is
$$
\widehat{Q}_M(x,y) = \frac{1}{M} \sum_{m=1}^{M} \mathscr{\hat A}_{\vec{\theta}}[f^{(m)}](x,y)
$$
Since $\mathscr{\hat{A}}_{\vec{\theta}}$ depends on the reduced basis matrices $\vec{U}_r$ and $\vec{\tilde{U}}_{\tilde{r}}$, the third case we consider is the reduced-dimensional Monte Carlo estimator, defined on the spatial grid for notational convenience:
$$
\vec{\widetilde Q}_M = \frac{1}{M} \sum_{m=1}^M \vec{\tilde{U}}_{\tilde{r}} \, \vec{\tilde{U}}_{\tilde{r}}^\top \, \ln(\vec{g}^{(m)}+1) \, .
$$

\begin{figure}
    \centering
    \includegraphics[width=0.5\textwidth]{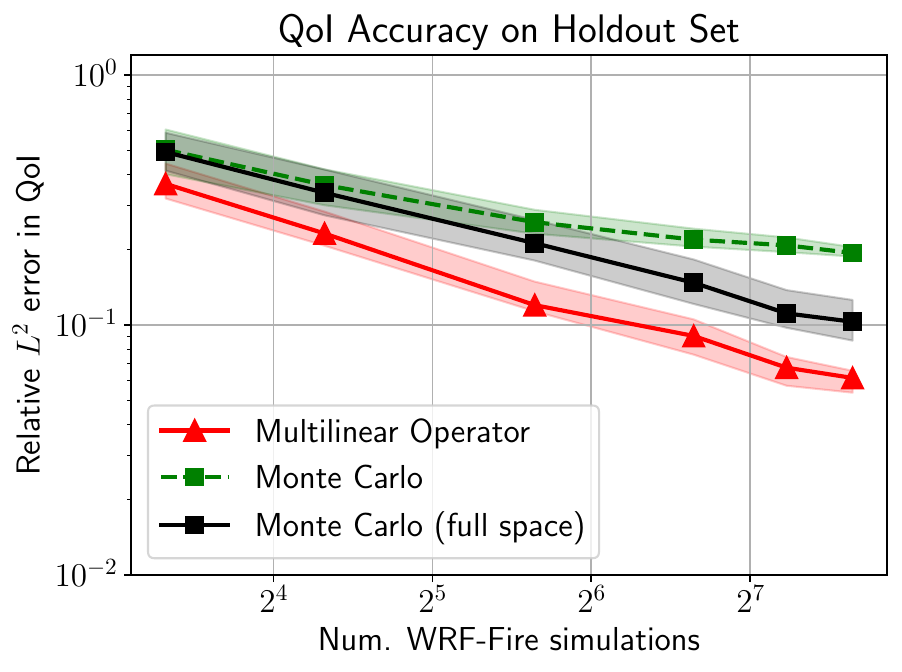}
    \caption{Three approaches for computing the quantity of interest~\cref{equ:qoi}. The shaded regions represent the 25\%--75\% range.}
    \label{fig:qoi}
\end{figure}

We show the results in \cref{fig:qoi}. We select $r$ and $\tilde{r}$ to capture 95\% of the snapshot standard deviation. Each Monte Carlo sample is a single high-fidelity model evaluation from the holdout set. In contrast, the multilinear operator is trained on a small number of training data and then applied to {\em all} members of the holdout set. For a given accuracy level, full-dimensional Monte Carlo sampling requires at least double the number of samples of a multilinear operator. \hl{Obtaining $10\%$ relative error requires roughly 100 more samples with Monte Carlo estimation than with the multilinear operator, representing 19 additional node-hours of computation.} Comparing reduced-dimensional Monte Carlo sampling to the multilinear operator isolates the projection error in~\cref{equ:pca_basis} from the least-squares error in~\cref{equ:lstsq}. In that case, reduced-dimensional Monte Carlo requires more than four times as many high-fidelity simulations to obtain the same accuracy as a multilinear operator. Relative error of $O(10^{-2})$ is acceptable because we expect the modeling error in empirical relationships between AOD and true power-output reduction to be larger than $O(10^{-2})$.

\subsection{QoI 2: Smoke Footprint}
\label{sec:classification}

Here, we study the spatial extent of a smoke plume using linear and quadratic bases for operators. In contrast to \cref{sec:monte-carlo}, we use fire and smoke snapshots for all observational times, not just those corresponding to the final state. Computing the SVDs~\cref{equ:svd} takes 24 seconds total. In both the linear and quadratic cases, we take $\tilde{r}$ to capture 95\% of the standard deviation in the training output snapshots $\vec{\tilde{S}}_\text{train}$, yielding $\tilde{r} = 273$.  This choice of $\tilde{r}$ yields 9\% relative error on validation outputs $\vec{\tilde{S}}_\text{val}$, and we examine other thresholds in \cref{sec:hyperparameter-sensitivity}. For convenience, we summarize the results of \crefrange{sec:linear-operator}{sec:gp} in \cref{tbl:classification-summary}.

\begin{table}
    \centering
    \begin{tabular}{|l|c|c|c|c|c|c|}
    \hline
    & Encoding & Training & Tuning & Hyperparams. & AUC & IoU \\ \hline 
    Linear operator & 24 sec & 0.02 sec & 146 sec & $r$, $\tilde{r}$, $\alpha$ & 0.93 & 0.61 \\ \hline 
    Quadratic operator & 24 sec & 29 sec & 436 sec & $r$, $\tilde{r}$, $\alpha$, $\lambda$ & 0.95 & 0.64 \\ \hline
    Gaussian Process (images) & --- & 0.56 sec & 70 sec & $l$, $\alpha$ & 0.83 & 0.45 \\ \hline
    Gaussian Process (PCA) & 24 sec & 0.07 sec & 147 sec & $l$, $r$, $\tilde{r}$, $\alpha$ & 0.68 & 0.35 \\ \hline
    \end{tabular}
    \caption{Summary of encoding, training, and hyperparameter tuning times for smoke footprint QoI, along with median classification accuracy. For reference, data generation dwarfs these times at 14,000 core-hours.}
    \label{tbl:classification-summary}
\end{table}

\subsubsection{Metrics}
\label{sec:metrics}

We consider the receiver operating characteristic (ROC), which is the curve in $\R^2$ traced by
$$
\text{ROC} = \left\{ (\text{FPR}(\alpha), \, \text{TPR}(\alpha)) : \alpha \geq 0 \right\}, \qquad \text{FPR} = \frac{\text{FP}}{\text{FP}+\text{TN}}, \qquad \text{TPR} = \frac{\text{TP}}{\text{TP} + \text{FN}},
$$
where $\alpha$ is a smoke concentration threshold value {\em in the model predictions}. Area under the curve (AUC) is the integral of $\text{TPR}(\alpha)$ as a function of $\text{FPR}(\alpha)$. A higher AUC indicates a more accurate classifier. At the threshold corresponding to the ROC nearest the point $(0,1)$, we compute the intersection over union (IoU), defined as
\begin{equation}
\text{IoU} = \frac{\text{TP}}{\text{TP} + \text{FP} + \text{FN}} \ .
\label{equ:iou}
\end{equation}
However, before computing these quantities, we must first define the presence of smoke {\em in the snapshot data}. We determine a threshold $\tau > 0$ from the cumulative distribution of total smoke at some cutoff level $\beta$:
$$
\tau = \argmax_{\tilde{\tau} > 0} \left( \frac{1}{M_{\text{val}}} \ \sum_{j=1}^{M_{\text{val}}} \ \frac{\mathds{1} \left\{ \left( \vec{\tilde S}_\text{val} \right)_{:, j} > \tilde{\tau} \right\}}{\mathds{1} \left\{ \left( \vec{\tilde S}_\text{val} \right)_{:,j} > 0 \right\}} > \beta \right) \,
$$
where $\mathds{1}$ is the indicator function applied componentwise. With $\beta=0.95$, we find $\tau = 0.158 \ \mu\text{g/m}^3$; that is, $\tau$ captures 95\% of the smoke in a snapshot, on average.

\subsubsection{Linear Operator}
\label{sec:linear-operator}

For the linear operator, we choose $r$ to capture 95\% of the standard deviation in $\vec{S}_{\text{train}}$, yielding $r=686$. This choice of $r$ yields 17\% relative error on validation inputs $\vec{S}_\text{val}$. There are $r\tilde{r} \approx 1.9 \times 10^5$ parameters total. Explicitly forming and solving the Gramian system~\cref{equ:gramian} takes 0.72 seconds, while computing $\vec{\theta}$ with \cref{thm:lstsq_pca_linear} takes 0.02 seconds, a speedup of two orders of magnitude. The relative $\ell^2$ training error is 13\%. To get model predictions, we first compute $\vec{U}_r^\top \vec{S}_\text{test}$, apply $\vec{\theta}$, and then lift back to the full-dimensional space with $\vec{\tilde{U}}_{\tilde{r}}$:
$$
\vec{Y}_\text{pred} = \max \left( \vec{\tilde{U}}_{\tilde{r}} \, \vec{\theta} \, \vec{U}_r^\top \vec{S}_\text{test}, \, \vec{0} \right),
$$
which takes 0.18 seconds, less than 0.05 ms per input. The relative test error in the Frobenius norm is
$$
\frac{\| \vec{Y}_\text{pred} - \vec{\tilde S}_\text{test} \|_{\text{F}}}{\| \vec{\tilde S}_\text{test} \|_{\text{F}}} = 29\%,
$$
which is large, but mostly driven by small spurious smoke plumes or overestimation of true smoke plume concentration. To justify this claim, we quantify the classification accuracy of $\vec{Y}^\text{pred}$ following \cref{sec:classification}. \Cref{fig:linear-classification} summarizes the classification performance on the test set (AUC = 0.93, IoU = 0.61), and \Cref{fig:linear-predictions} shows model test predictions versus observations at different IoU percentiles. The predictions in \cref{fig:linear-predictions} are visually quite similar, but some predictions do exhibit spurious plumes or slightly elevated concentrations compared to the observational data. The colorbars hint at a possible correlation between smoke footprint size and IoU, which we show in \Cref{fig:iou-vs-smoke}, \hl{suggesting that classification is more accurate on large smoke plumes}.

\begin{figure}
    \centering
    \includegraphics[width=0.99\textwidth]{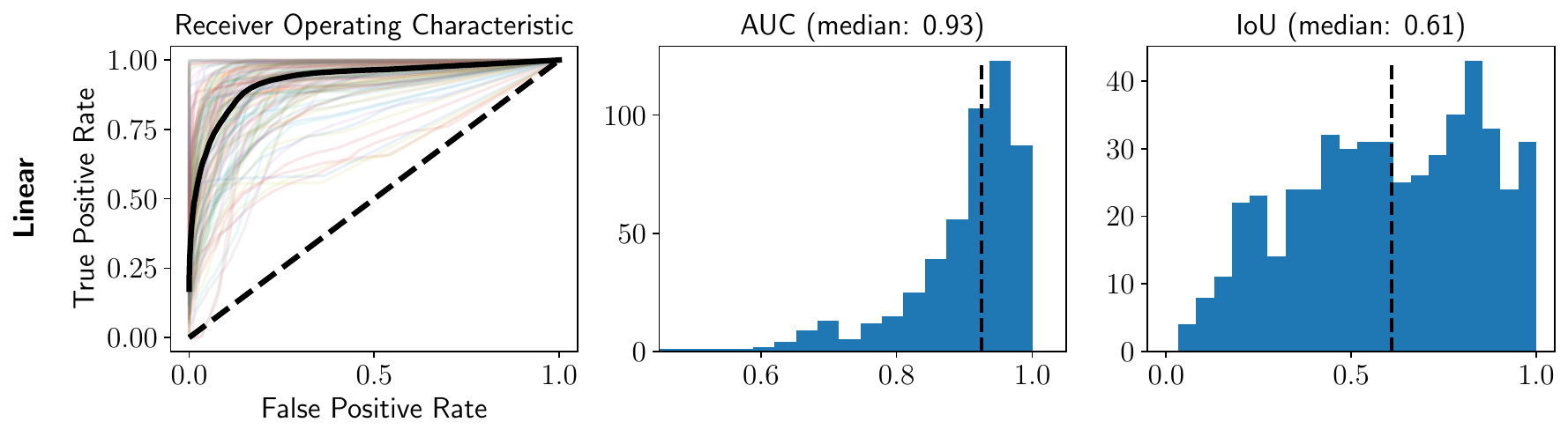}
    \caption{Classification performance of linear model. The solid black line is the pointwise median ROC. To reduce clutter, only a subset of ensemble members are shown in the left panel.}
    \label{fig:linear-classification}
\end{figure}

\begin{figure}
    \centering
    \includegraphics[width=0.99\textwidth]{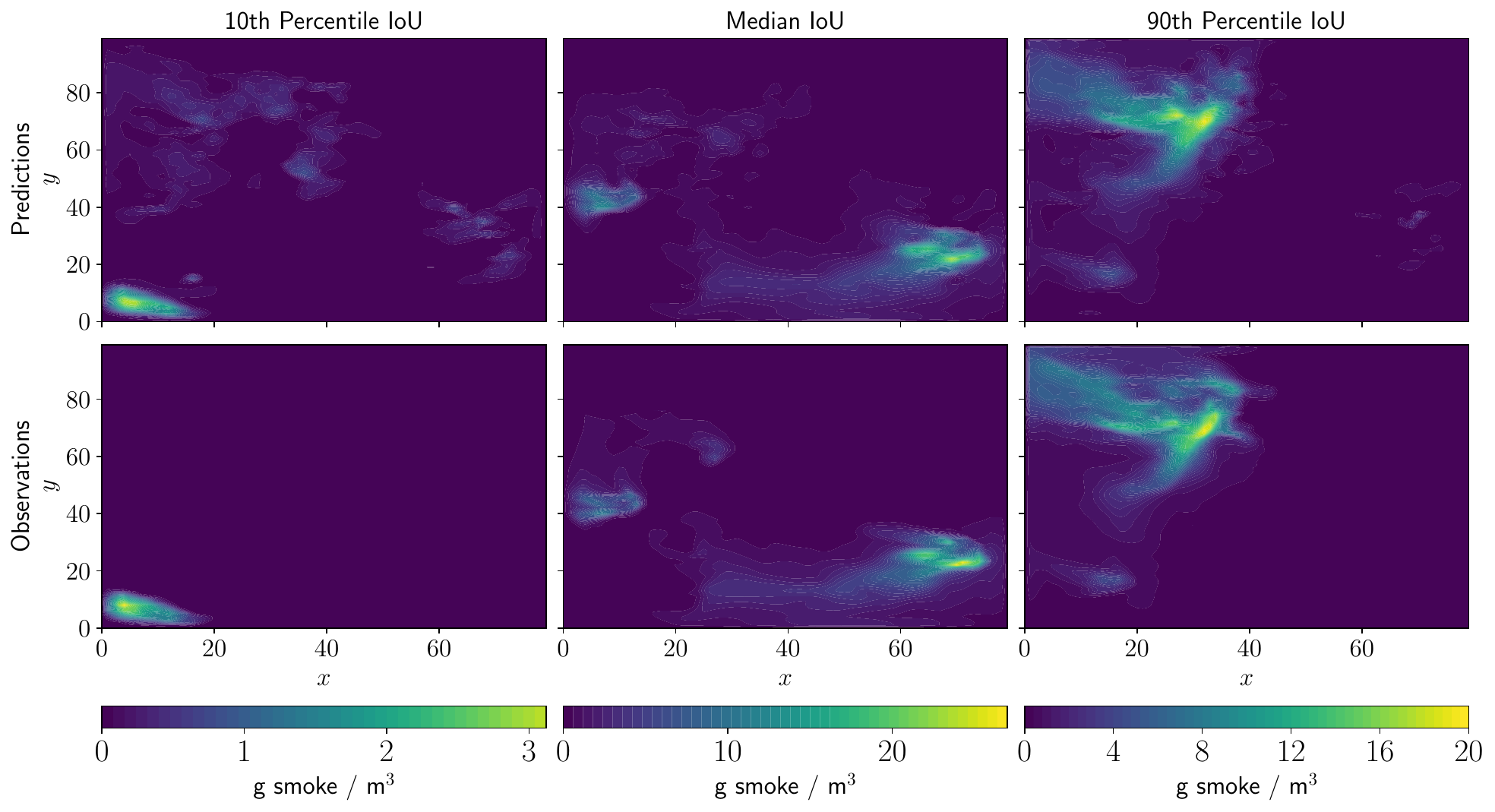}
    \caption{Linear model predictions and observations at different IoU percentiles. The rows in a given column share the same colormap.}
    \label{fig:linear-predictions}
\end{figure}

\begin{figure}
    \centering
    \includegraphics[width=0.33\textwidth]{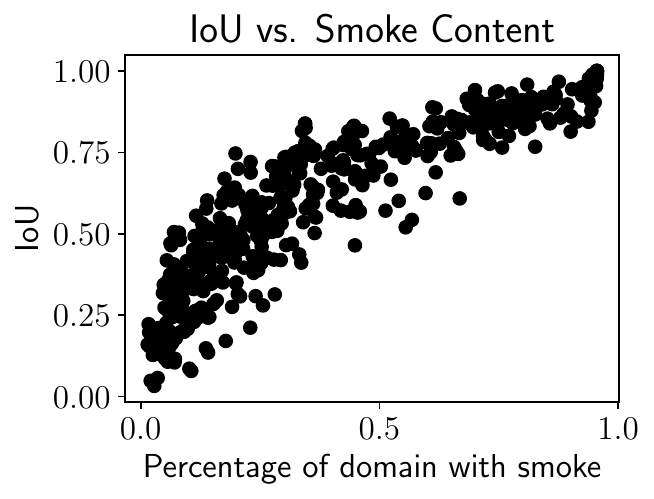}
    \caption{Correlation of IoU with smoke footprint extent.}
    \label{fig:iou-vs-smoke}
\end{figure}

\subsubsection{Quadratic Operator}
\label{sec:quadratic-operator}

In the quadratic case, we choose $r=200$ since the column dimension of $\vec{M}$ is $O(r^2)$. $\vec{M}$ has 20300 columns, so there are approximately $1.1 \times 10^7$ trainable parameters total. The projection error on the validation inputs is 24\%. Forming $\vec{M}_\text{train}$ takes 6 seconds. Solving~\cref{equ:lstsq_pca_quadratic} with $\vec{M} = \vec{M}_\text{train}$ takes 23 seconds. Forming and solving the least-squares system is significantly more time-consuming than the linear case, which takes 0.02 seconds. We use $\lambda=10^5$, which we determined with hyperparameter tuning (see \Cref{fig:hyperparameter-sensitivity}). Computing
$$
\vec{Y}_\text{pred} = \max \left( \vec{\tilde{U}}_{\tilde{r}} \, \vec{\theta} \, \vec{M}_\text{test}, \, \vec{0} \right)
$$
takes 0.48 seconds, which is less than 0.15 ms per input. The median relative $\ell^2$ error is 40\%. As before, \Cref{fig:quadratic-classification} summarizes holdout-set classification performance (AUC = 0.95, IoU = 0.65), which is slightly higher than the linear case. \Cref{fig:quadratic-predictions} shows model predictions versus observations at different IoU percentiles, which are visually quite similar.

\begin{figure}
    \centering
    \includegraphics[width=0.99\textwidth]{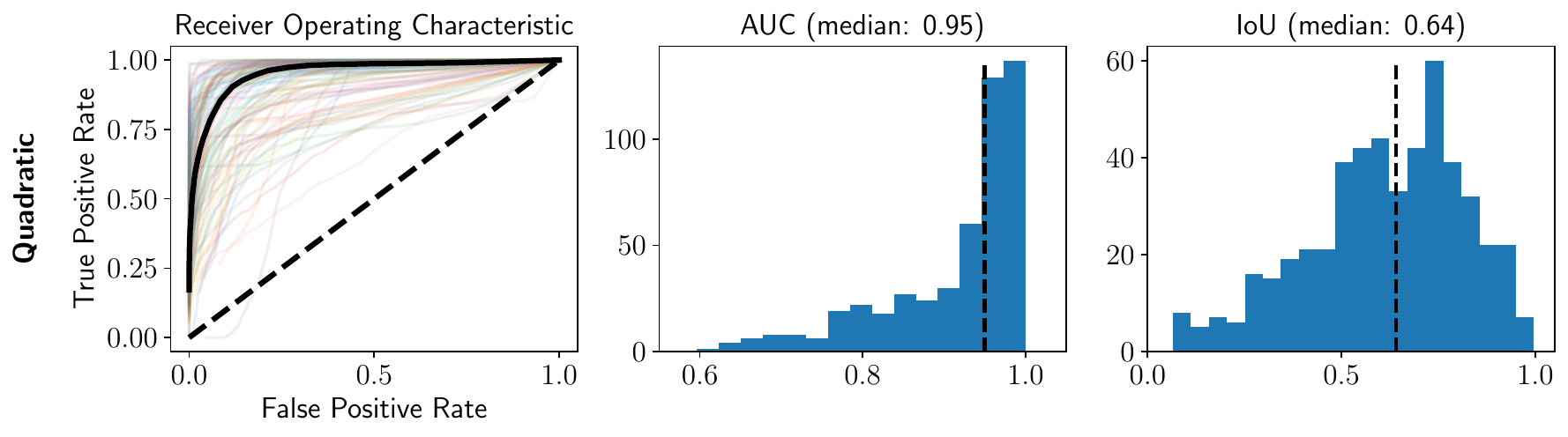}
    \caption{Classification performance of quadratic model. The solid black line is the pointwise median ROC. To reduce clutter, only a subset of ensemble members are shown in the left panel.}
    \label{fig:quadratic-classification}
\end{figure}

\begin{figure}
    \centering
    \includegraphics[width=0.99\textwidth]{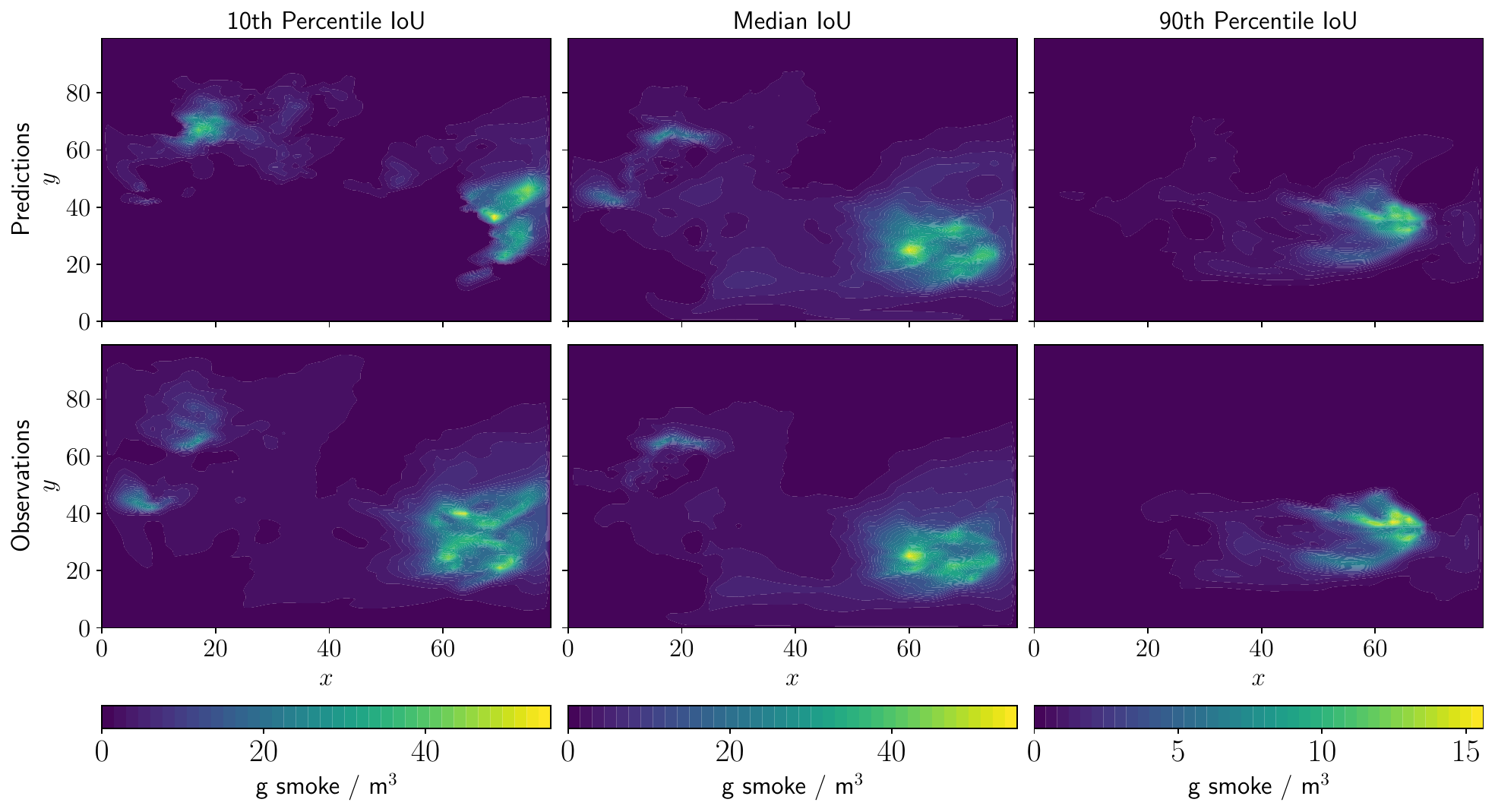}
    \caption{Quadratic model predictions and observations at different accuracy percentiles. The rows in a given column share the same colormap.}
    \label{fig:quadratic-predictions}
\end{figure}

\subsubsection{Hyperparameter Sensitivity}
\label{sec:hyperparameter-sensitivity}

Here, we investigate the impact of truncation levels on ultimate model accuracy. In \crefrange{sec:linear-operator}{sec:quadratic-operator}, we used 95\% of the snapshot standard deviation, 95\% of the total smoke, and $\lambda=10^5$. Now we vary each of these hyperparameters to capture differing amounts of variance and smoke. We show the resulting classification accuracy in \cref{fig:hyperparameter-sensitivity} on the validation set. As the ratio of captured smoke increases, there is a slight decrease in AUC since more smoke locations than the tallest peaks get included. Accuracy degrades slightly as more variance is retained in the PCA expansion, which leads to the multilinear operator learning possibly spurious relationships in the training set. The regularization parameter $\lambda$ has minimal impact.

\begin{figure}
    \centering
    \includegraphics[width=0.99\textwidth]{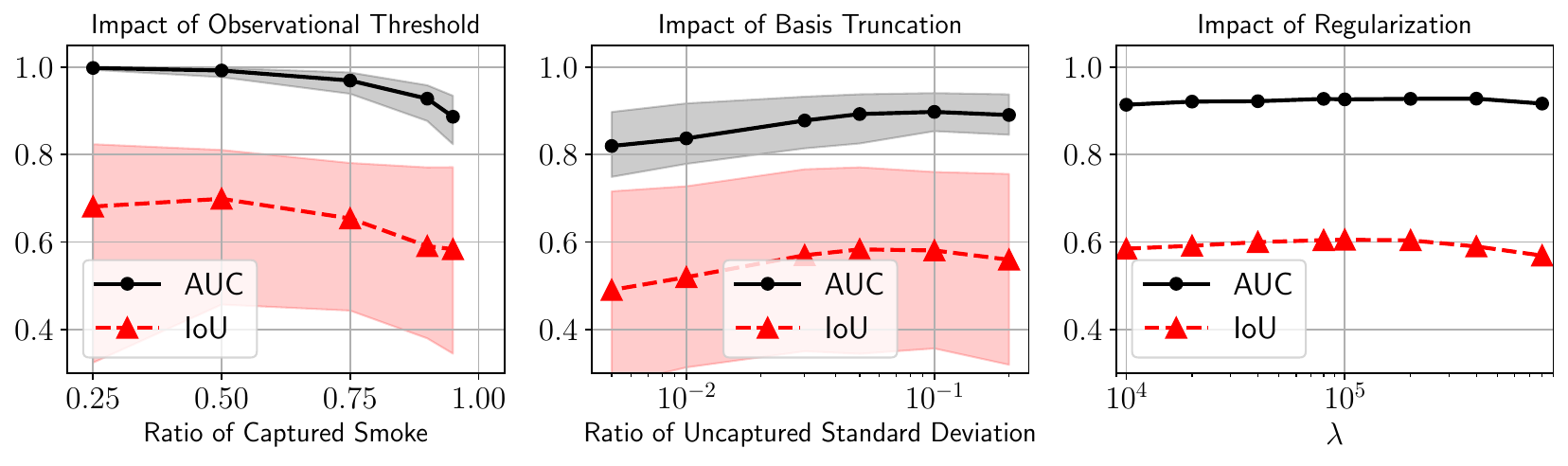}
    \caption{Left: Changing observational threshold $\alpha$ (linear operator). Center: Changing PCA accuracy (linear operator). Right: Changing regularization (quadratic operator). All quantities are computed on the validation set, not the holdout set.}
    \label{fig:hyperparameter-sensitivity}
\end{figure}

\subsubsection{Distributional Shift}
\label{sec:distributional-shift}

An important consideration in environmental modeling is how a model will perform under distributional shift, when conditions that produced training data are no longer satisfied at inference time. To investigate this, we consider low, medium, and high drought conditions, corresponding to different weather and fuel moisture fractions within WRF-Fire. We train a linear operator on only medium-drought conditions and perform classification under low- and high-drought conditions, where the smoke threshold still captures $\beta=95\%$ of total smoke. \Cref{fig:distributional-shift} shows that distributional shift degrades, but does not destroy, classification accuracy. Unsurprisingly, the median and dispersion of AUC and IoU are now worse than with in-distribution prediction, where the median AUC and IoU are 0.93 and 0.61, respectively. However, median AUC is still above 0.5 for each case; AUC for low-drought prediction is 0.63. IoU is approximately 0.3 in each case, which is acceptable but unremarkable performance for a classifier; we discuss this further in \cref{sec:discussion}. The peaks at 1.0 correspond to the case where there is neither observational nor predicted smoke, since the observational threshold is determined from out-of-distribution data.

\begin{figure}
    \centering
    \includegraphics[width=0.99\textwidth]{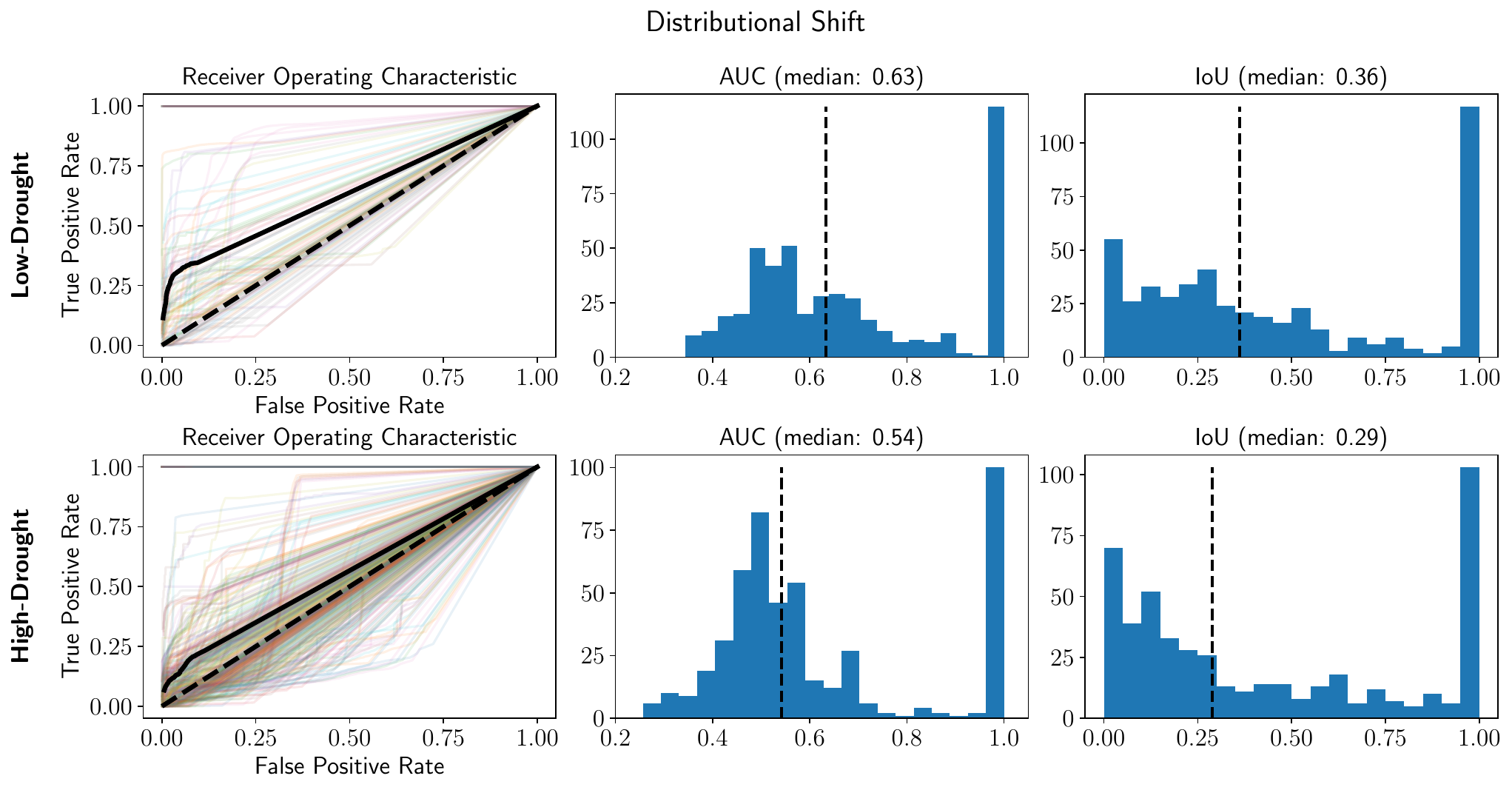}
    \caption{Classification accuracy on different distributions than were seen in training.}
    \label{fig:distributional-shift}
\end{figure}

\subsubsection{Gaussian Processes}
\label{sec:gp}
We now benchmark against a concentration field constructed with a Gaussian Process (GP). In this setting, we do not perform any basis transformation and instead assume a fixed discretization, learning the map directly between images. We then apply the same analysis as before to quantify classification performance.

GPs are standard and highly popular regression techniques. For a function sampled from a GP, any finite number of point-evaluations is normally distributed with mean and covariance given by mean and kernel functions, respectively~\cite{rasmussen2006gaussian}. We use squared-exponential Mat\'{e}rn kernel with length-scale $l$ and variance $\sigma^2$:
$$
k(\vec{a}, \vec{b}; \, l) = \exp\left( - \frac{\|\vec{a} - \vec{b}\|_2^2}{2l^2} \right) \, .
$$
where $l$ will be optimized with L-BFGS. We use a GP to learn the map $h : \R^{8000} \to \R^{8000}$, representing flattened $100 \times 80$ images of time-since-ignition and smoke concentration. We also consider a GP that maps input coefficients ($r=686$) to output coefficients ($\tilde{r} = 273$). To maintain parity with the training times of multilinear operators, we randomly subsample 100 input--output pairs for training. \Cref{fig:gp-classification} displays the results. For fixed image sizes, there is a median AUC of 0.83 and a median IoU of 0.45. Even though there is significantly less data, training is still 25 times slower than with the linear model, although the tuning process is faster as there are fewer hyperparameters. For PCA coefficients, the AUC and IoU are lower, at 0.68 and 0.35 respectively; since the input and output dimensions are much smaller, tuning $l$ takes only 0.5 seconds here. In either case, classification accuracy is lower than either the linear or quadratic operator, which have a median AUC and IoU of at least 0.93 and 0.61, respectively.

\begin{figure}
    \centering
    \includegraphics[width=0.99\textwidth]{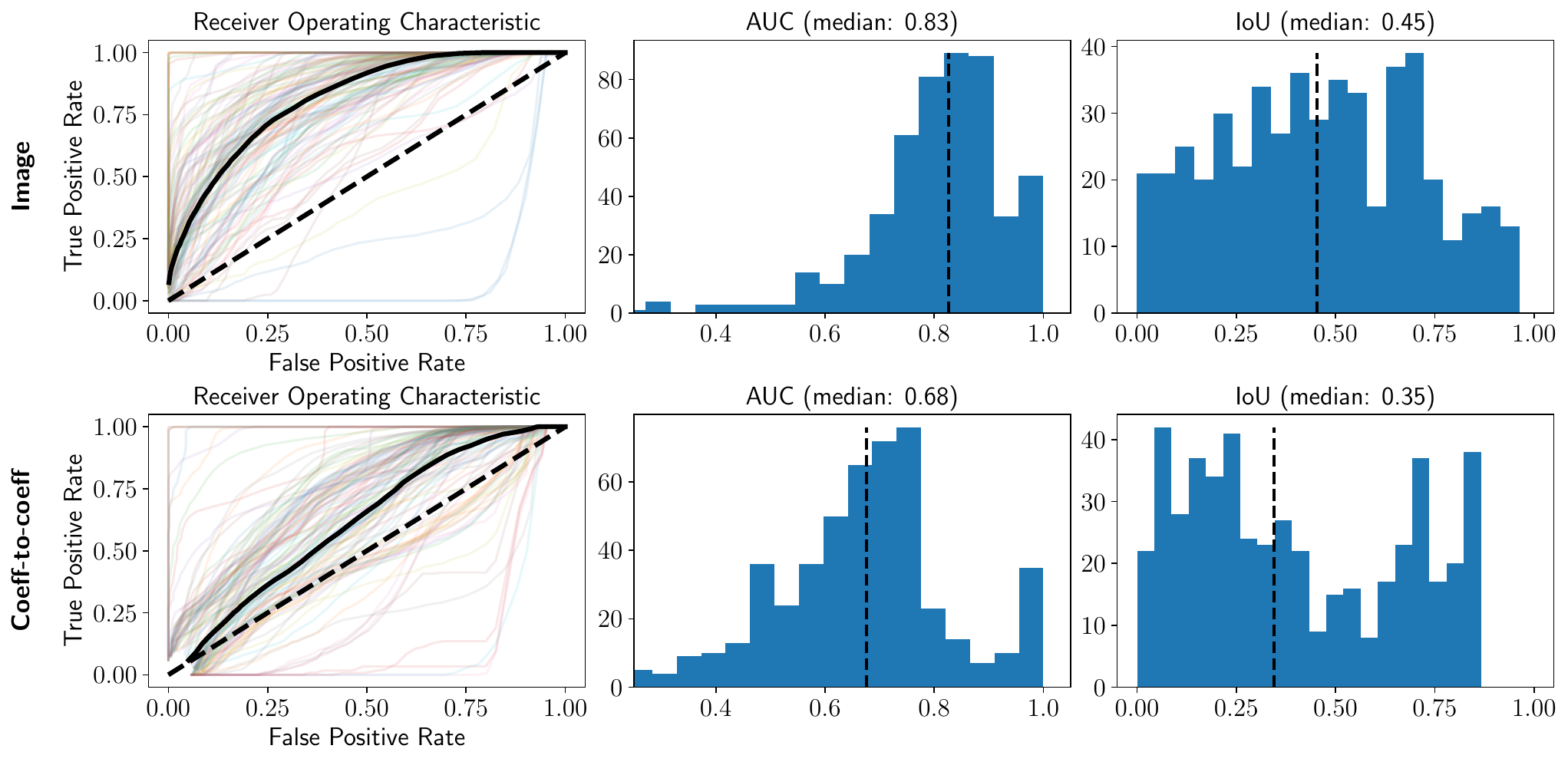}
    \caption{Classification performance using a GP to map between input and output images of fixed size (top) and input--output coefficients (bottom).} 
    \label{fig:gp-classification}
\end{figure}

\subsection{Discussion}
\label{sec:discussion}

Our results indicate that multilinear operators are effective surrogate models for fire-to-smoke maps, positioning them as a promising tool for many-query outer-loop analysis. The computational cost of our method, expressed as runtime, is also quite low, as shown in \Cref{tbl:classification-summary}. \Crefrange{fig:linear-classification}{fig:quadratic-predictions} show that median AUC is at least 0.9 for both the linear and quadratic models, indicating highly effective classifiers. This suggests that sum-of-squares error is overly pessimistic for determining the presence of smoke above some critical value. Further evidence is the close visual similarity of the prediction plots in \Cref{fig:linear-predictions} and \Cref{fig:quadratic-predictions}. However, \Cref{fig:hyperparameter-sensitivity} does indicate better performance as less smoke is captured, i.e. as the footprint excludes smaller plumes, which is also consistent with the correlation in \Cref{fig:iou-vs-smoke}. Although hyperparameter tuning is costly, \cref{fig:hyperparameter-sensitivity} shows little benefit gained from costly tuning efforts.

\cref{sec:distributional-shift} indicates that our method suffers under distributional shift, which causes it to perform as an essentially uninformative classifier. We believe this is because, as drought conditions increase, the fuel becomes dryer and emits proportionally more smoke, but our method does not specifically incorporate fuel moisture content as an input (for speed of training), which is a limitation we intend to address in future work. \hl{However, the purpose of the current study is to construct a surrogate fire-to-smoke map for tractable Monte Carlo smoke analysis, not to perform causal attribution of model performance on each possible physical driver.} Even with distributional shift, the IoU of our method is still superior to the state-of-the-art. Furthermore, the weather and fuel-management assumptions in our outer-loop analysis will be prescribed up front, so we will be able to obtain training data that covers the full range of scenarios.

\subsubsection{Comparison with Baselines}
\label{sec:baselines}

Previous work on spatially dependent smoke detection~\cite{larsen2021deep} claims an ``IoU'' of 57\% but uses a nonstandard definition of IoU. The authors treat the absence of smoke as a ``true'' signal and then average with the usual IoU \cref{equ:iou}. This approach results in overly optimistic model predictions and obfuscates the point of IoU: to determine classification accuracy when the signal is far less common than background noise. Indeed, their Supplemental Material shows the relevant IoU \cref{equ:iou} as 0.15, in Figure S3. For timings, the most similar prior work~\cite{larsen2021deep} does not provide training time, although a ML-based wildfire forecasting model \cite{lahrichi2025improved} reported average training and inference times of 1.25 hours and 7.7 ms, respectively, on a GPU. Furthermore, a GP results in lower performance than either a linear or quadratic multilinear operator, as shown in \Cref{tbl:classification-summary} and \Cref{fig:gp-classification}, and tuning $l$ can require significant time even with a training set six times smaller.

%% file: conclusion.tex
\section{Conclusion and Future Work} \label{sec:conclusion}

We have presented a method for approximating fire-to-smoke maps with highly efficient operators. We use PCA to compute a representation of the inputs (time-since-ignition) and outputs (mean vertical smoke concentration). We then map from inputs to outputs using a linear combination of powers of the PCA input coefficients, cf. \cite{turnage2025optimal}. In the linear case, \Cref{thm:lstsq_pca_linear} provides a closed-form solution for the operator, which can be computed at negligible cost after two SVDs. In terms of AUC and IoU, our operators perform substantially better than the previous state-of-the-art. The reduced cost of this method, as well as the enhanced accuracy demonstrated in \Crefrange{fig:linear-classification}{fig:quadratic-predictions}, indicates that it can be successfully applied to outer-loop analysis of fuel succession for the Upper Rio Grande Watershed. Fire-to-smoke operators for other areas can be built with site-specific training data. \Cref{sec:results} contains all hyperparameter settings required for reproducibility, while \Cref{sec:appendix} ensures conformity with best practices for trustworthy scientific machine learning.

In future work, we will train multilinear operators in an offline phase in order to quantify smoke impacts quickly in an online phase with Monte Carlo sampling. We will run many LANDIS-II simulations in parallel with various weather conditions, fuel-treatment regimes, and time horizons. LANDIS-II also has a simple fire model that quickly provides a low-fidelity time-since-ignition field. We will then use a multilinear operator, trained on high-fidelity smoke transport, to map LANDIS fires to smoke concentration. In this setting, different weather may have a significant impact on model performance (\cref{sec:distributional-shift}), requiring probabilistic model calibration against observations. \hl{The final step will leverage empirical relationships between smoke concentration, aerosol optical depth, and solar power production to quantify energy output under various fuel treatment strategies.}

Additionally, this surrogate model outputs the passive total smoke concentration, though specific emissions may be more important in other use-cases. Characterizing wildfire smoke as its constituent elements requires understanding of (i) what is burned, (ii) the amount, rate, and behavior of consumption, and (iii) how microphysics and chemistry within the plume alter emissions over time. Our methodology is capable of developing separate surrogates for each emission, e.g. from a well-parameterized WRF-Chem model. Ongoing work in the Fire and Smoke Model Evaluation Experiment \cite{prichard2019fire} would assist in generating these emission parameters.

%% file: appendices.tex
\appendix

\definecolor{darkgreen}{RGB}{80,200,120}
\newcommand{\greencheck}{{\bf \color{darkgreen} \checkmark}}
\newcommand{\redx}{{\bf \color{red} $\bm \times$}}

\section{Scientific Machine Learning Best Practices}
\label[app]{sec:appendix}

Here, we specifically document sixteen recommendations for trustworthy scientific machine learning~\cite[p.~2]{jakeman2025verification}.

\begin{enumerate}
    \itemsep0em
    \item Problem Definition
    \begin{enumerate}
        \itemsep0em
        \item Specify prior knowledge and model purpose (\greencheck, \cref{sec:intro})
        \item Specify verification, calibration, validation, and application domains (\greencheck, \cref{sec:results})
        \item Carefully select and specify quantities of interest (\greencheck, \crefrange{sec:monte-carlo}{sec:classification})
        \item Select and document model structure (\greencheck, \cref{sec:formulation})
    \end{enumerate}
    \item Verification
    \begin{enumerate}
        \itemsep0em
        \item Verify code implementation with idealized test problems ($\greencheck$, \cite{jakeman2023pyapprox})
        \item Verify solution accuracy with realistic benchmarks ($\greencheck$, \cref{fig:linear-classification}, \cref{fig:linear-predictions}, \cref{fig:quadratic-classification}, \cref{fig:quadratic-predictions}, and \cite{turnage2025optimal})
    \end{enumerate}
    \item Validation
    \begin{enumerate}
        \itemsep0em
        \item Perform probabilistic calibration (N/A, no inverse problem to solve)
        \item Validate model against purpose-specific requirements (\greencheck, \crefrange{sec:linear-operator}{sec:quadratic-operator})
        \item Quantify prediction uncertainties (\greencheck, \cref{fig:linear-classification}, \cref{fig:quadratic-classification})
    \end{enumerate}
    \item Continuous Credibility Building
    \begin{enumerate}
        \itemsep0em
        \item Document data characteristics and impact (\greencheck, \cref{sec:background}, \cref{sec:setup})
        \item Document data processing procedures (\greencheck, \cref{sec:wrf-chem}, \cref{ssec:pca}, \cref{sec:setup})
        \item Quantify SciML model sensitivities (\greencheck, \cref{sec:hyperparameter-sensitivity})
        \item Document the hyperparameter selection process (\greencheck, \cref{sec:hyperparameter-sensitivity})
        \item Use software testing and ensure reproducibility (\greencheck~for algorithm~\cite{jakeman2023pyapprox} and data~\cite{morrow2026enablingzenodo})
        \item Compare developed SciML model against alternatives (\greencheck, \cref{sec:baselines})
        \item Explain the SciML prediction mechanism (\greencheck, \cref{sec:formulation}, \cref{sec:metrics})
    \end{enumerate}
\end{enumerate}